\documentclass[twoside]{article}
\usepackage[accepted]{aistats2018}

\usepackage{url}            
\usepackage{booktabs}       
\usepackage{amsfonts}       
\usepackage{nicefrac}       
\usepackage{microtype}      
\usepackage{natbib}
\usepackage{comment, color}
\usepackage{amsthm}
\usepackage{amsmath}
\usepackage{algorithm, algorithmicx, algpseudocode}
\usepackage{appendix}
\usepackage{graphicx, caption, subcaption}
\graphicspath{ {figs/} }

\newtheorem{theorem}{Theorem}

\newtheorem{lemma}{Lemma}

\newtheorem{proposition}{Proposition}

\begin{document}
\runningauthor{Grover et al.}
\twocolumn[
\aistatstitle{Best arm identification in multi-armed bandits with delayed feedback}

\aistatsauthor{Aditya Grover$^{\ast 1}$, Todor Markov$^1$, Peter Attia$^1$, Norman Jin$^1$, Nicholas Perkins$^1$, Bryan Cheong$^1$,}
\aistatsauthor{Michael Chen$^1$, Zi Yang$^2$, Stephen Harris$^3$, William Chueh$^1$, Stefano Ermon$^1$}
\aistatsaddress{$^1$Stanford University \;\; $^2$University of Michigan \;\; $^3$Lawrence Berkeley National Laboratory}
]

\begin{abstract}
We propose a generalization of the best arm identification problem in stochastic multi-armed bandits (MAB) to the setting where every pull of an arm is associated with \textit{delayed} feedback. The delay in feedback increases the effective sample complexity of standard algorithms, but can be offset if we have access to \textit{partial} feedback received before a pull is completed. We propose a general framework to model the relationship between partial and delayed feedback, and as a special case we introduce efficient algorithms for settings where the partial feedback are biased or unbiased estimators of the delayed feedback. Additionally, we propose a novel extension of the algorithms to the \textit{parallel} MAB setting where an agent can control a batch of arms. Our experiments in real-world settings, involving policy search and hyperparameter optimization in computational sustainability domains for fast charging of batteries and wildlife corridor construction, demonstrate that exploiting the structure of partial feedback can lead to significant improvements over baselines in both sequential and parallel MAB.
\end{abstract}
\section{INTRODUCTION}
Intelligent agents often need to interact with the environment and make rational decisions that optimize for a suitable objective. One such setting that commonly arises is the best arm identification problem in stochastic multi-armed bandits~\citep{bubeck2009pure,audibert2010best}. 
In a multi-armed bandit (MAB) problem, an agent is given a set of $n$ finite actions (or arms), each associated with a reward drawn from an arm-specific probability distribution. In a pure exploration setting, the goal is to reliably identify the top-$k$ arms while minimizing the exploration cost. This problem has numerous applications, including optimal experimental design.

We consider a new variant of this problem where the feedback rewards are received after a delay. Delayed feedback is common in the real-world. For instance, hypothesis testing in science and engineering often suffers from delayed feedback since they involve expensive, time-consuming experiments. In one of the motivating applications of this work
we want to search over fast-charging policies for electrochemical batteries to maximize lifetime, overcoming the difficulties posed due to lengthy experiments.
Even within the field of machine learning, finding the best hyperparameter settings for a given learning algorithm and dataset can be modeled as a best arm identification problem involving a non-trivial delay~\citep{jamieson2016non}.  

However, many scenarios of interest are not complete black-boxes during the intermediate time steps before receiving a delayed feedback reward. 
Depending on the application, we often have access to side-information in the form of \textit{partial feedback} that can aid decision making. These could be extra measurements such as temperature and remaining capacity while charging batteries in the aforementioned scenario, or learning curves for hyperparameter optimization. 

In this work, we propose a general-purpose framework for modeling delayed feedback in MAB, and take a deeper dive into several practically relevant instantiations.
In particular, we design and analyze algorithms for best arm identification in the fixed confidence setting where the partial feedback are biased or unbiased estimators of the delayed feedback. 
Our proposed algorithms adaptively tune the mean and confidence estimates 
wherever the partial feedback reduces the overall uncertainty.
We also extend these algorithms to the parallel MAB setting where we are allowed to pull a batch of arms at every time step~\citep{jun2016top}.

Finally, we empirically validate the proposed algorithms on simulated data and real world datasets drawn from two domains. The first corresponds to experimental design for finding the optimal charging policy for a battery that maximizes overall lifetime~\citep{moura2017battery}. In the second domain, we perform hyperparameter optimization for finding the best cut strategy for a standard mixed integer programming solver with performance tested on a benchmark set of problem instances drawn from computational sustainability~\citep{gomes2008connections}. Our experiments demonstrate that accounting for partial feedback can reduce the delayed sample complexity on average by 15.6\% and 80.8\% for sequential MAB over baselines for the two application scenarios respectively. The corresponding average savings over baselines for parallel MAB are 20.7\% and 87.6\% respectively. 

\section{BACKGROUND \& MODELING FRAMEWORK}\label{sec:background}
The chief workhorse of our analysis will be the law of iterated logarithms (LIL) that analyzes the limiting behavior of random walks (sequence of pulls for a given arm in our case) defined over sub-Gaussian random variables~\citep{darling1967iterated}. Several finite LIL bounds have been proposed in the literature; we consider the one proposed by~\citet{zhao2016adaptive} which has been shown to outperform others empirically while retaining the same asymptotic behavior. Alternate bounds, such as the one by~\cite{jamieson2014lil}, could also be used with no effect on the theoretical analysis of this work.

\begin{lemma}\label{thm:lil_adaptive}
Let $X^{(1)}, X^{(2)}, \ldots$ be i.i.d. sub-Gaussian random variables with scale parameter $\sigma$ and mean $\mu$. Let $\tau$ be any random variable with domain $\mathbb{N}$. For any $c > 1, 2a > c, b>0$, the following holds with probability at least $1-2\zeta(\nicefrac{2a}{c})e^{\nicefrac{-2b}{c}}$:
\begin{align*}
\left \vert \frac{1}{\tau}\sum_{l=1}^\tau X^{(l)} - \mu \right \vert \leq \sigma \sqrt{\frac{a \log (\log_c \tau +1) + b }{\tau}} 
\end{align*}
\end{lemma}
where $\zeta$ denotes the Riemannian zeta function.
The constants in Lemma~\ref{thm:lil_adaptive} are chosen such that the lemma holds for a target confidence.  To simplify the notation, we denote the the error probability by $\delta'$ and the right hand side of Lemma~\ref{thm:lil_adaptive} by $C\left(\sigma, \tau, \delta'\right)$ such that the following holds with probability $1-\delta'$ for any $\tau \in \mathbb{N}$:
\begin{align}\label{eq:lil}
\left \vert \frac{1}{\tau}\sum_{l=1}^\tau X^{(l)} - \mu \right \vert \leq C\left(\sigma, \tau, \delta'\right).
\end{align}

We consider a stochastic multi-armed bandit (MAB) problem characterized  by a set of $n$ arms, indexed by $i = 1, \ldots, n$. Each arm is associated with a fixed, unknown probability distribution with means $\{\mu_i\}_{i=1}^n$. We assume that the means are unique. Without loss of generality, assume that the arm indices are sorted as per the means, such that $\mu_1 > \mu_2 > \ldots > \mu_n$. 

We are interested in the pure exploration setting, also known as the best arm identification problem, where the goal of an agent is to identify the top-$k$ arms (with the highest means) with a target confidence $1-\delta$ while minimizing the total time spent on exploration. Exploration in our setting, however, is not the same across the pulls of a given arm.
In particular, we assume that each pull of an arm is associated with an unknown (stochastic) delay that contributes to the total exploration time. The presentation in this section assumes a \textit{sequential} MAB setting where the agent can pull/run only one arm at a given time step; the alternate \textit{parallel} MAB setting where an agent can control a ``batch'' of arms at once is discussed in Section~\ref{sec:parallel}~\citep{perchet2015batched,wu2015identifying,jun2016top}.

Formally, the stochastic data generating process with delayed feedback can be described as follows. At any given start time $t_s$: 
\begin{enumerate}
\item Agent chooses an arm $i$.
\item Nature samples a delay $D_s \geq 1$ from an (unknown) arm specific delay distribution.
\item Nature samples a sequence of partial feedback, $(Y_{i,t_s+1}, \ldots, Y_{i,t_s+D_s}) \mid D_s$ jointly. The joint distribution of the partial feedback depends on $\mu_i$.\\\\
In general, the delay and partial feedback sequence are unknown to the agent at time $t_s$. 
\end{enumerate}
At time $t_s+\Delta$ where $\Delta \in [1, D_s]$,
\begin{enumerate}
\setcounter{enumi}{3}
\item Nature reveals $Y_{i,t_s+\Delta}$ to the agent.\\
If $\Delta=D_s$, the agent goes to step 1. Otherwise, the agent decides whether to continue the current pull (step 4) or start another pull (step 1) in which case any remaining partial feedback for the current pull will not be observed.
\end{enumerate}

The agent and nature continue to play the above game until the agent has selected a set of candidate top-$k$ arms. 
The delay $D_s$ can contribute significantly to the total time spent on exploration. Under appropriate assumptions however, we can exploit the structure in the partial feedback to significantly reduce the overall exploration cost of delayed feedback. The data generating process described above is very general and one can make many natural assumptions on the distribution of the partial feedback $(Y_{i,t_s+1}, \cdots, Y_{i,t_s+D_s}) \mid D_s$. 

For instance, we can model the following scenarios:
\begin{itemize}
\item \textbf{Full delayed feedback}: The partial feedback at the last delay, $Y_{i,t_s+D_s}$ is sub-Gaussian with mean $\mu_i$ and scale parameter $\sigma_i$. For the intermediate time steps, $\Delta \in [1, D_s-1]$, we have $Y_{i,t_s+\Delta} =0 $, and hence, we receive no information about $\mu_i$ at these time steps.
\item \textbf{Incremental partial feedback}: The set of partial feedback $Y_{i,t_s+\Delta}$ for every time step $\Delta \in [1, D_s]$ consists of mutually independent, sub-Gaussian random variables with mean $\nicefrac{\mu_i}{D_s}$ and scale parameter $\nicefrac{\sigma_i}{\sqrt{D_s}}$.  
Hence, the cumulative partial feedback $\sum_{\Delta=1}^{D_s} Y_{i,t_s+\Delta}$ is also sub-Gaussian with mean $\mu_i$ and scale parameter $\sigma_i$.
\item \textbf{Unbiased noisy partial feedback:} The partial feedback at the last delay, $Y_{i,t_s+D_s}$ is sub-Gaussian with mean $\mu_i$ and scale parameter $\sigma_i$. For the intermediate time steps, $\Delta \in [1, D_s-1]$, the set of partial feedback 
$Y_{i,t_s+\Delta}\mid Y_{i,t_s+D_s} - Y_{i,t_s+D_s}$ consists of mutually independent, sub-Gaussian random variables with zero mean and scale parameter $\sigma^{(p)}_{i}$.
\item \textbf{Biased noisy partial feedback:} The partial feedback at the last delay, $Y_{i,t_s+D_s}$ is sub-Gaussian with mean $\mu_i$ and scale parameter $\sigma_i$. For the intermediated time steps, $\Delta \in [1, D_s-1]$, the set of partial feedback 
$Y_{i,t_s+\Delta}\mid Y_{i,t_s+D_s} - Y_{i,t_s+D_s}$ consists of mutually independent, sub-Gaussian random variables with mean $b_i$ and scale parameter $\sigma^{(p)}_{i}$. 
Here, $b_i$ is a fixed, but unknown bias associated with the partial feedback for the arm.
\end{itemize}
 
Note that the standard MAB setting where we observe the feedback at the immediate next time step is a special case of the full delayed feedback with a constant delay $D_s=1$ for every pull. In fact, the algorithms for best arm identification in the \textit{full delayed} and \textit{incremental partial feedback} settings can be derived naturally from the standard MAB algorithms with no delays. Specifically, the agent can simply chose to ignore the time instants at which delayed feedback is unavailable for the full delayed feedback setting. The sample complexity of any such algorithm is hence the number of arm pulls required in the standard MAB setting weighted by the delay of every pull.  
These settings are still interesting for parallel MAB where information can be shared across arms; we discuss this case in Section~\ref{sec:parallel}.

The \textit{partial feedback} settings, however, present an interesting scenario where the agent can extract information from noisy feedback.
For such settings, we propose modified algorithms based on racing-style procedures typically used for the standard MAB setting~\citep{maron1994hoeffding}. Typically, racing algorithms maintain three disjoint arm sets: accepted arms $A$, rejected arms $R$, and surviving arms $S$. Initially, all arms are assigned to the surviving set $S$. Racing procedures uniformly sample arms while removing them from the surviving set based on confidence bounds. For convenience, define the lower confidence bounds (LCB) 
and upper confidence bounds (UCB) for every arm $i$ as: 
\begin{align}
LCB_i := \widehat{\mu}_i - C_i\label{eq:lcb}\\
UCB_i := \widehat{\mu}_i + C_i \label{eq:ucb}
\end{align}
where $\widehat{\mu}_i$ is the empirical mean of the feedback for arm $i$ and the confidence bound $C_i$ will depend on the particular racing algorithm under consideration.  Let $k_t:=k-\vert A \vert$ be the effective number of top arms remaining to be identified at a time step $t$.
Each time we receive a feedback reward (full or partial), the racing procedures update these sets based on the rule that any arm in $S$ whose LCB is greater than the UCB of $\vert S \vert -k_t$ arms is accepted. Similarly, any arm in $S$ whose UCB is less than the LCB of $k_t$ arms is rejected. The racing procedure is repeated until $S$ is empty. The pseudocode for the subroutine that updates the arm sets is given in Algorithm~\ref{alg:racingsubroutines}.

\begin{algorithm}[t]  \caption{RacingSubroutines}
   \label{alg:racingsubroutines} 
\begin{algorithmic}
\Function{$\mathrm{UpdateArmSets}$}{arm sets $A$, $R$, $S$, top $k$, confidence bounds $\{LCB_i, UCB_i\}_{i \in S}$}
\State Initialize $k_t \leftarrow k - |A|$.
\State Update \resizebox{.8\hsize}{!}{$A \leftarrow A \cup \{i \in S \mid LCB_i > \max_{j \in S}^{(k_t + 1)} UCB_j\}$}.
\State Update \resizebox{.8\hsize}{!}{$R \leftarrow R \cup \{i \in S \mid UCB_i < \max_{j \in S}^{(k_t)} LCB_j\}$}.
\State Update $S \leftarrow S \backslash \{ R\cup A\}$.
\State \Return $A$, $R$, $S$.
\EndFunction
\State	\Function{$\mathrm{GetBatchArms}$}{surviving arms $S$, counts $\{N_i, a_i\}_{i \in S}$, effective batch size $e$, limit $r$}
\State Initialize new arm pulls $\mathbf{m}\leftarrow \mathbf{0} \in \mathbf{R}^n$.
\For {slot $s = \left\{1, \cdots, \min\left(e, \vert S \vert r\right)\right\}$}
\State Least pulled arm $j \leftarrow \text{arg } \min_{i \in S: a_i \leq r} N_i$
\State Update $a_j \leftarrow a_j + 1$.
\State Update $m_j \leftarrow m_j + 1$.
\State Update $N_j \leftarrow N_j + 1$.
\EndFor 
\State \Return $\mathbf{m}, \{N_i\}_{i \in S}, \{a_i\}_{i \in S}$
\EndFunction
\end{algorithmic}
\end{algorithm}

\section{SEQUENTIAL MAB}\label{sec:single}
In sequential MAB, we assume that the agent can receive (partial) feedback from only a single arm pull at any given time step, \textit{e.g.}, we can only perform one experiment at a time. We skip a separate discussion on the trivial full feedback (and the related incremental feedback) setting and discuss it only in the context of the \emph{noisy feedback settings}. For convenience, we denote the partial feedback at the last delay as $X_{i,t_s}=Y_{i,t_s+D_s}$. Here, $X_{i,t_s}$ is a sub-Gaussian random variable with mean $\mu_i$ and scale parameter $\sigma_i$. The proofs of all results in this section are given in the Appendix.

\subsection{Unbiased noisy partial feedback}
In this setting, an agent has access to unbiased partial feedback at the intermediate time steps before receiving the full delayed feedback. In the following result, we derive a variation of the finite LIL bound for the unbiased partial feedback setting.

\begin{proposition}\label{thm:lil_noisy}
Let $\{Y_{i,t_1+1}, Y_{i,t_1+2}, \ldots, Y_{i,t_1+D_1}$, $Y_{i,t_2+1}, \ldots, Y_{i,t_2+D_2}, \ldots\}$ denote the partial feedback sequences
for the pulls of an arm $i$ started at time steps $t_{1}, t_{2}, \ldots$ and delays $D_1, D_2, \ldots$. Then, under the distributional assumptions on the unbiased partial feedback (see Section~\ref{sec:background}) for any $F\in \mathbb{N}$, $P \in [1, D_F]$, $\delta_f > 0, \delta_p > 0$, we have with probability $1-\delta_f-\delta_p$:
\begin{align}\label{eq:lil_noisy_1}
&\left \vert \frac{1}{F}\left[\sum_{f=1}^{F-1} X_{i,t_f} +  \frac{1}{P}\sum_{l=1}^P Y_{i, t_{F}+l} \right] - \mu_i  \right \vert  \nonumber \\
&\leq C\left(\sigma_i, F, \nicefrac{\delta_f}{n}\right) + \frac{1}{F}C\left(\sigma^{(p)}_i, P, \nicefrac{\delta_p}{n}\right) \forall i \in [1, n]
\end{align}
\end{proposition}
where $X_{i,t_f} = Y_{i,t_f+D_f}$ by definition. At any intermediate time step between the the start and end of the $F$-th arm pull, Proposition~\ref{thm:lil_noisy} adaptively ``splits" the confidence bounds pertaining to the full delayed feedback for $F$ steps (first term in the RHS) and the partial delayed feedback for the $F$-th arm pull (second term in the RHS). Contrast this with the full delayed feedback setting where the following confidence bound holds with probability $1-\delta$:
\begin{align}\label{eq:lil_full_delayed}
\left \vert \frac{1}{F-1}\sum_{f=1}^{F-1} X_{i,t_f} - \mu_i \right \vert \leq C\left(\sigma_i, F-1, \nicefrac{\delta}{n}\right) \forall i \in [1, n]
\end{align}
To obtain the same target confidence in the two cases above, we constrain $\delta = \delta_f + \delta_p$. Solving for the optimal $\delta_f^\ast, \delta_p^\ast$ that minimize the RHS of Eq.~\eqref{eq:lil_noisy_1} under the constraint due to $\delta$ corresponds to a convex optimization problem that can be solved in closed form.
Comparing the mean estimators in Eq.~\eqref{eq:lil_noisy_1} and Eq.~\eqref{eq:lil_full_delayed}, we note that the agent can only use the full delayed feedback up till the $(F-1)$-th arm pull while waiting for the outcome of the $F$-th arm pull in the latter case while the former dynamically incorporates the partial feedback observed for the $F$-th arm pull. 
\begin{algorithm}[t]
   \caption{RacingUnbiasedPF (arm parameters $\{i, \sigma_i, \sigma^{(p)}_i\}_{i=1}^n$, top $k$, confidence $\delta$)}
   \label{alg:seqracingnoisypf}
\begin{algorithmic}[1]
\State Initialize global time step $t=0$, surviving $S=\{i\}_{i=1}^n$, accepted $A=\{\}$, rejected $R=\{\}$.
\State Initialize per-arm full delayed feedback counter $F_i=0$, empirical means $\hat{\mu}_{i}=0$,  confidence bounds $LCB_i=-\infty$,  $UCB_i=\infty$ for all $i \in S$. 
\While{$S$ is not empty}
\While {$\mathrm{True}$}
\State Increment $t \leftarrow t+1$.
\State Collect partial feedback $Y_{a, t}$.
\State Update $\widehat{\mu}^{(p)} \leftarrow \frac{(P \widehat{\mu}^{(p)} + Y_{a, t})}{(P+ 1)}$. 
\State Increment $P \leftarrow P + 1$. 
\State Set \resizebox{.72\hsize}{!}{$C^{(partial)} \leftarrow C(\sigma_a, F_a + 1, \nicefrac{\delta_f^\ast}{n}) + \frac{C(\sigma^{(p)}_a, p, \nicefrac{\delta_p^\ast}{n})}{F_a + 1}$.}
\State\label{line:partial_start} Choose \resizebox{.65\hsize}{!}{$\mathrm{FOrP} \leftarrow \mathrm{arg}\min \left(C(\sigma_a, F_a, \nicefrac{\delta}{n}), C^{(partial)}\right)$.}
\State Update \resizebox{.65\hsize}{!}{$C_a \leftarrow C(\sigma_a, F_a, \nicefrac{\delta}{n})$ if $\mathrm{FOrP}=F$ else $C^{(partial)}$}.
\State Update \resizebox{.65\hsize}{!}{$\widehat{\mu}_a \leftarrow  \widehat{\mu}^{(f)}$ if $\mathrm{FOrP}=F$ else $\frac{F_a\widehat{\mu}^{(f)} + \widehat{\mu}^{(p)}}{F_a + 1}$}.
\State\label{line:partial_end} Update $LCB_a, UCB_a$.
\State\label{line:racing_elimination} \resizebox{.8\hsize}{!}{$A, R, S \leftarrow \mathrm{UpdateArmSets}(A, R, S, k, \{LCB_i, UCB_i)\}_{i \in S})$}.
\If {$P=D_{a, t_a}$ or $a \not \in S$}\label{line:end_delay} 
\State Break \Comment{Pull on  termination/elimination}
\EndIf
\EndWhile 
\State\label{line:get_new_arm} Pull arm $a$ where $a \leftarrow \mathrm{arg} \min_{a \in S} F_a$.
\State Initialize start $t_a\leftarrow t$, partial feedback counter $P = 0$, partial mean $\widehat{\mu}^{(p)} = 0$, full mean $\widehat{\mu}^{(f)} \leftarrow \widehat{\mu}_i$.
\EndWhile
\State \Return $A$
\end{algorithmic}
\end{algorithm}

\begin{algorithm}[t]
   \caption{BatchRacingFullDF(arm parameters $\{i, \sigma_i\}_{i=1}^n$, top $k$, confidence $\delta$, batch $b$, limit $r$)}
   \label{alg:batchracingfulldf}
\begin{algorithmic}[1]
\State Initialize global time step $t=0$, pull status counts $\mathrm{running}=0$, surviving arms $S=\{i\}_{i=1}^n$, accepted arms $A=\{\}$, rejected arms $R=\{\}$.
\State Initialize per-arm global pull counts $N_i=0$, running pull counts $a_i=0$, full delayed feedback $F_i=0$, empirical means $\hat{\mu}_{i}=0$, confidence bounds $LCB_i=-\infty$, $UCB_i=\infty$ for all $i \in S$. 
\While{$S$ is not empty}
\If {$\mathrm{running}>0$}
\State Increment $t \leftarrow t+1$.
\State Collect batch full delayed feedback $Y$.
\ForAll {$Y_{h, t} \in Y$}
\State Update $\widehat{\mu}_h \leftarrow \nicefrac{(F_h \widehat{\mu}^{(f)} + Y_{h, t})}{(F_h + 1)}$. 
\State Increment $F_h \leftarrow F_h+1$. 
\State Update $LCB_h, UCB_h$.
\State Decrement $a_h \leftarrow a_h-1$.
\EndFor
\If {$Y$ is not empty}\label{line:updracingstart}
\State \resizebox{.75\hsize}{!}{$A, R, S \leftarrow \mathrm{UpdateArmSets}(A, R, S, k, \{LCB_i, UCB_i\}_{i \in S})$}.
\State Decrement $\mathrm{running} \leftarrow \mathrm{running} - \vert Y\vert$.
\label{line:updracingend}\EndIf
\EndIf
\State\label{line:constraintstart} Update arms $\mathbf{m}$, counts $\{N_i, a_i\}_{i \in S} \leftarrow \mathrm{GetBatchArms}(S, \{N_i, a_i\}_{i \in S}, b-\mathrm{running}, r)$.
\State\label{line:constraintend} Pull every arm $j \in \mathbf{m}\; m_j$ times.
\State Update $\mathrm{running}\leftarrow \mathrm{running} + \sum_{j \in \mathbf{m}} m_j$.
\EndWhile
\State \Return $A$
\end{algorithmic}
\end{algorithm}

Based on the above analysis, we propose a racing algorithm for the unbiased partial feedback setting with the psuedocode given in Algorithm~\ref{alg:seqracingnoisypf}. 
At any intermediate time step, the agent chooses a mean estimator and a confidence bound for the current arm (Lines~\ref{line:partial_start}-\ref{line:partial_end}). The choice corresponds to the tighter confidence bound obtained either by optimizing Eq.~\eqref{eq:lil_noisy_1} over $\delta_p, \delta_f$ or the one obtained by Eq.~\eqref{eq:lil_full_delayed} where only the full delayed feedback are considered. Thereafter, the agent invokes the racing subroutine that checks whether a surviving arm can be rejected or accepted (Line~\ref{line:racing_elimination}). If the pull has finished running or the current arm is itself eliminated (Line~\ref{line:end_delay}), the agent pulls a new arm in the next time step which has the least number of full delayed feedback (Line~\ref{line:get_new_arm}). 

We can make some observations about Algorithm~\ref{alg:seqracingnoisypf}. First, we see that an agent adopting the proposed algorithm can never do worse than the alternate racing strategy that considers estimates only based on the full delayed feedback. This is because even at the intermediate time steps, the agent considers the mean estimator corresponding to the smaller of the two confidence bounds, which can only reduce the delayed sample complexity of the algorithm. Whenever an arm pull has finished, the agent also updates the mean and confidence interval by an arithmetic averaging over $only$ the full delayed feedback. Using partial feedback is impractical at such time steps since the partial feedback only introduce noise and do not provide any additional information about the true mean.

If the maximum possible delay associated with any arm pull is given by $D_{\max}$, then we can trivially extend bounds for the sample complexity of racing style procedures~\citep{jamieson2014best} to derive similar bounds on the \textit{delayed sample complexity} with an extra multiplicative factor of $D_{\max}$.\footnote{The delayed sample complexity for an algorithm refers to the total number of time steps (including delays) before termination.} This is similar to what one would expect from the full delayed feedback setting and is not surprising for Algorithm~\ref{alg:seqracingnoisypf} 
since in the absence of any additional assumptions, the partial feedback could be completely uninformative and the algorithm will choose to ignore them. We believe 
domain-specific assumptions about the delay distribution and the noise associated with the partial feedback as a function of time could lead to a tighter analysis and is an interesting direction of future work. The correctness of Algorithm~\ref{alg:seqracingnoisypf} can be summarized below.
\begin{theorem}\label{thm:seqracingnoisypf_sample}
Assuming the delay associated with any arm pull is bounded, then 
Algorithm~\ref{alg:seqracingnoisypf} outputs the top-$k$ arms with probability at least $1-\delta$.
\end{theorem}

To get further intuition about the working of Algorithm~\ref{alg:seqracingnoisypf}, consider the situation where all arms have been pulled once except one. When the last remaining arm is pulled for the first time, the full delayed feedback setting will necessarily have to wait for the pull to finish running before eliminating the arms whereas Algorithm~\ref{alg:seqracingnoisypf} can potentially start eliminating arms right after the first partial delayed feedback is received. 

\begin{figure*}[t]
\centering
\begin{subfigure}[b]{0.31\textwidth}
\centering
\includegraphics[width=\textwidth]{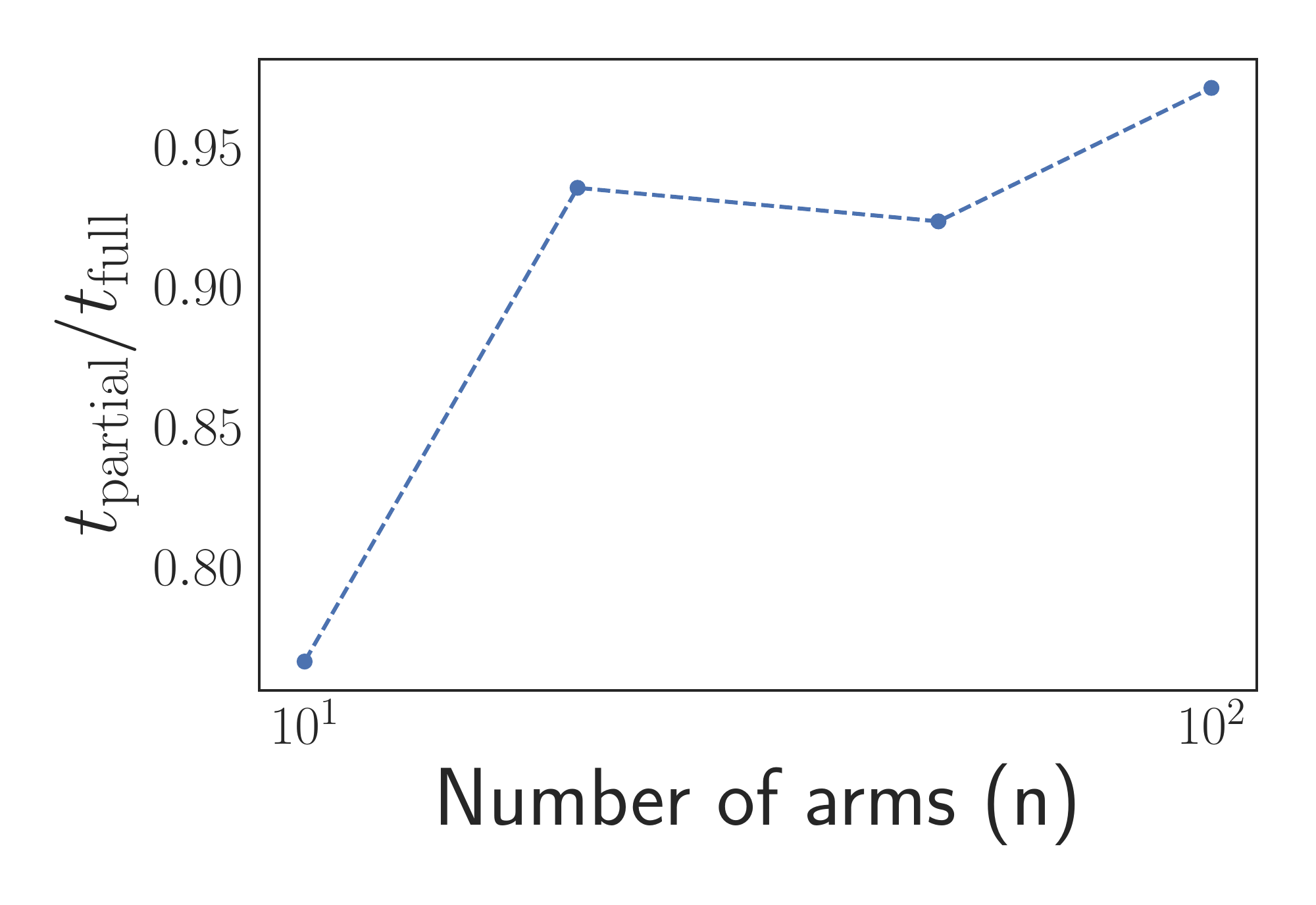}
\end{subfigure}
~
\begin{subfigure}[b]{0.295\textwidth}
\centering
\includegraphics[width=\textwidth]{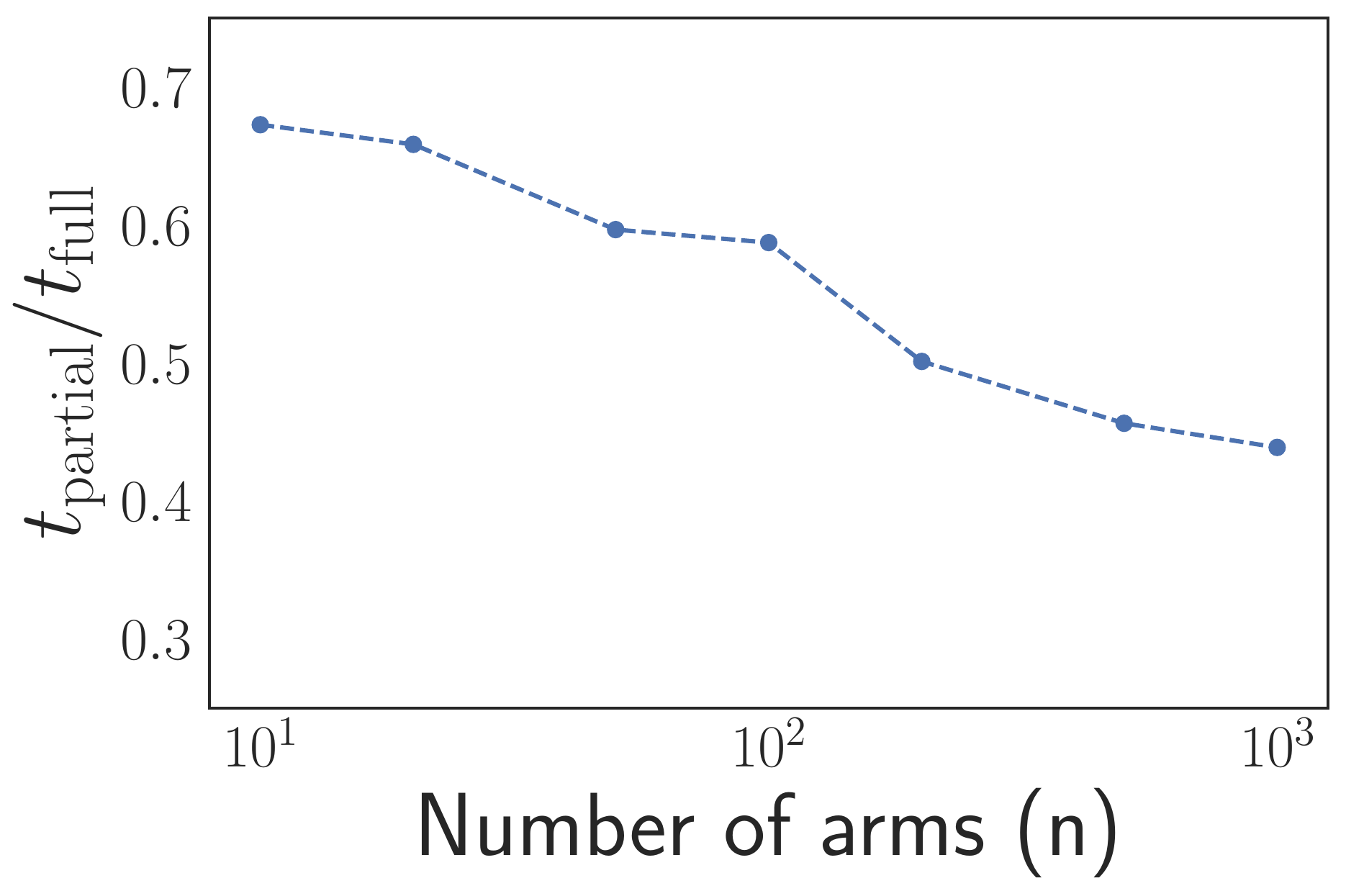}
\end{subfigure}
~
\begin{subfigure}[b]{0.315\textwidth}
\centering
\includegraphics[width=\textwidth]{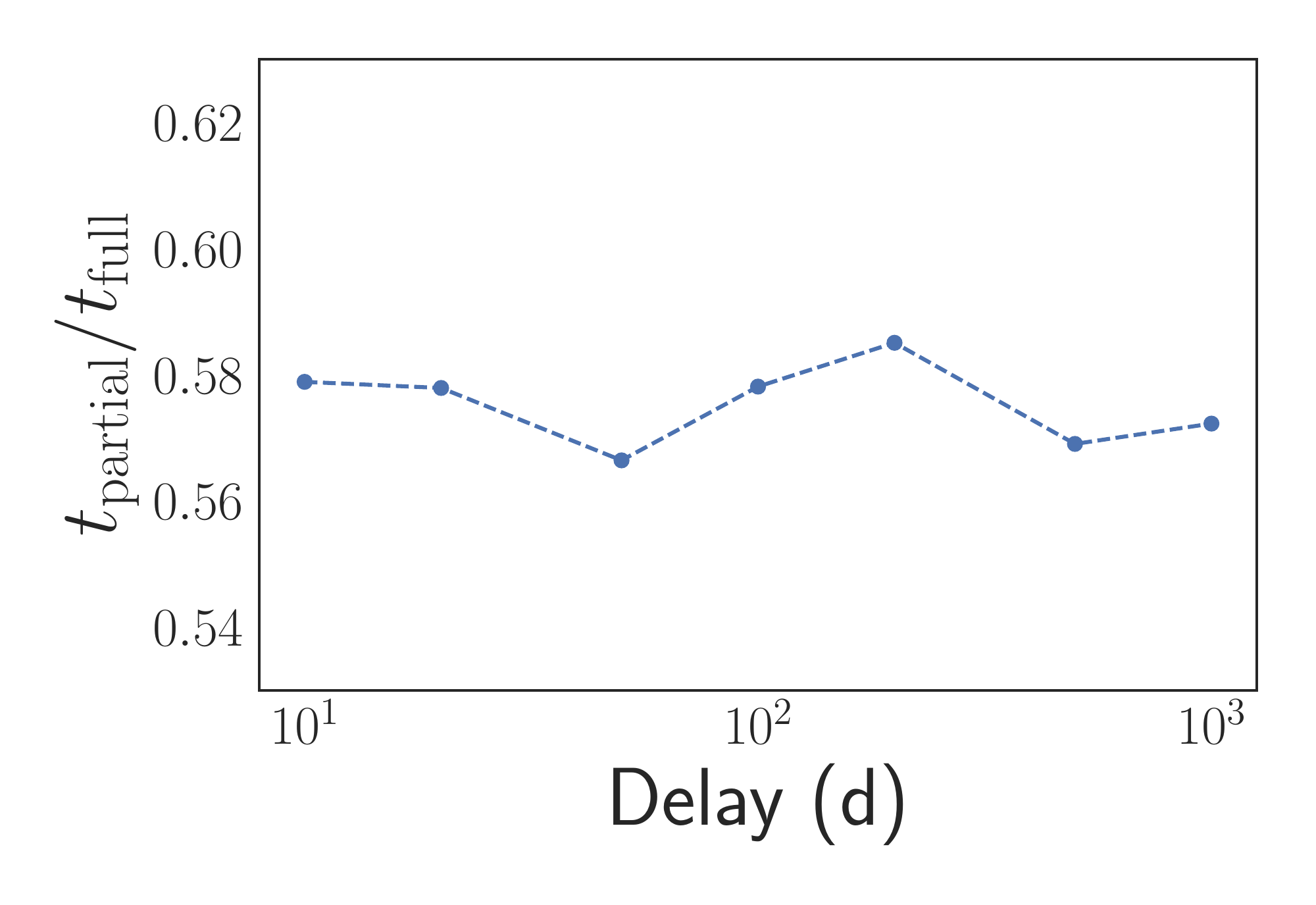}
\end{subfigure}

\begin{subfigure}[b]{0.31\textwidth}
\centering
\includegraphics[width=\textwidth]{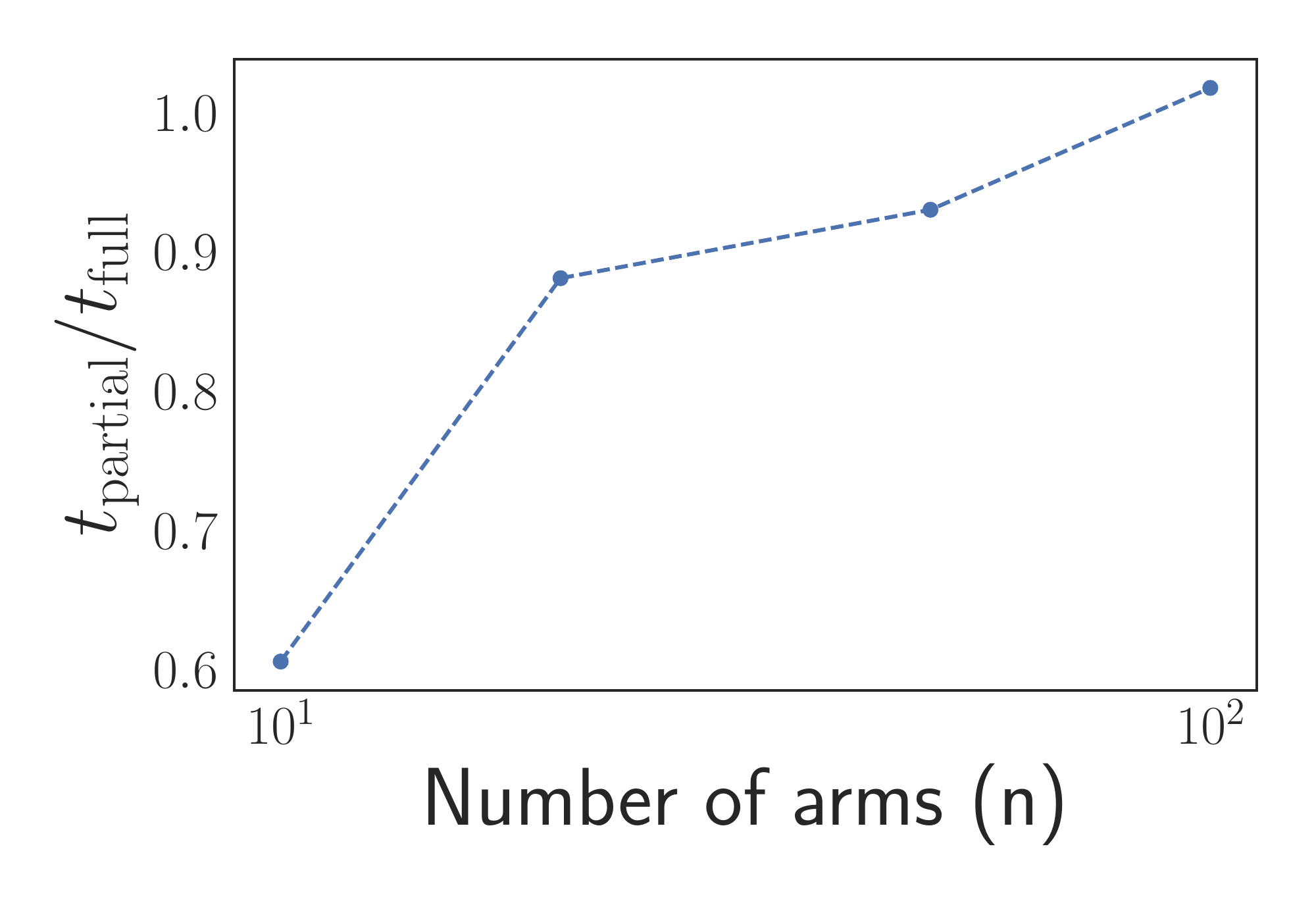}
\caption{Number of arms - Bounded means}\label{fig:exp_bounded_means}
\end{subfigure}
~
\begin{subfigure}[b]{0.295\textwidth}
\centering
\includegraphics[width=\textwidth]{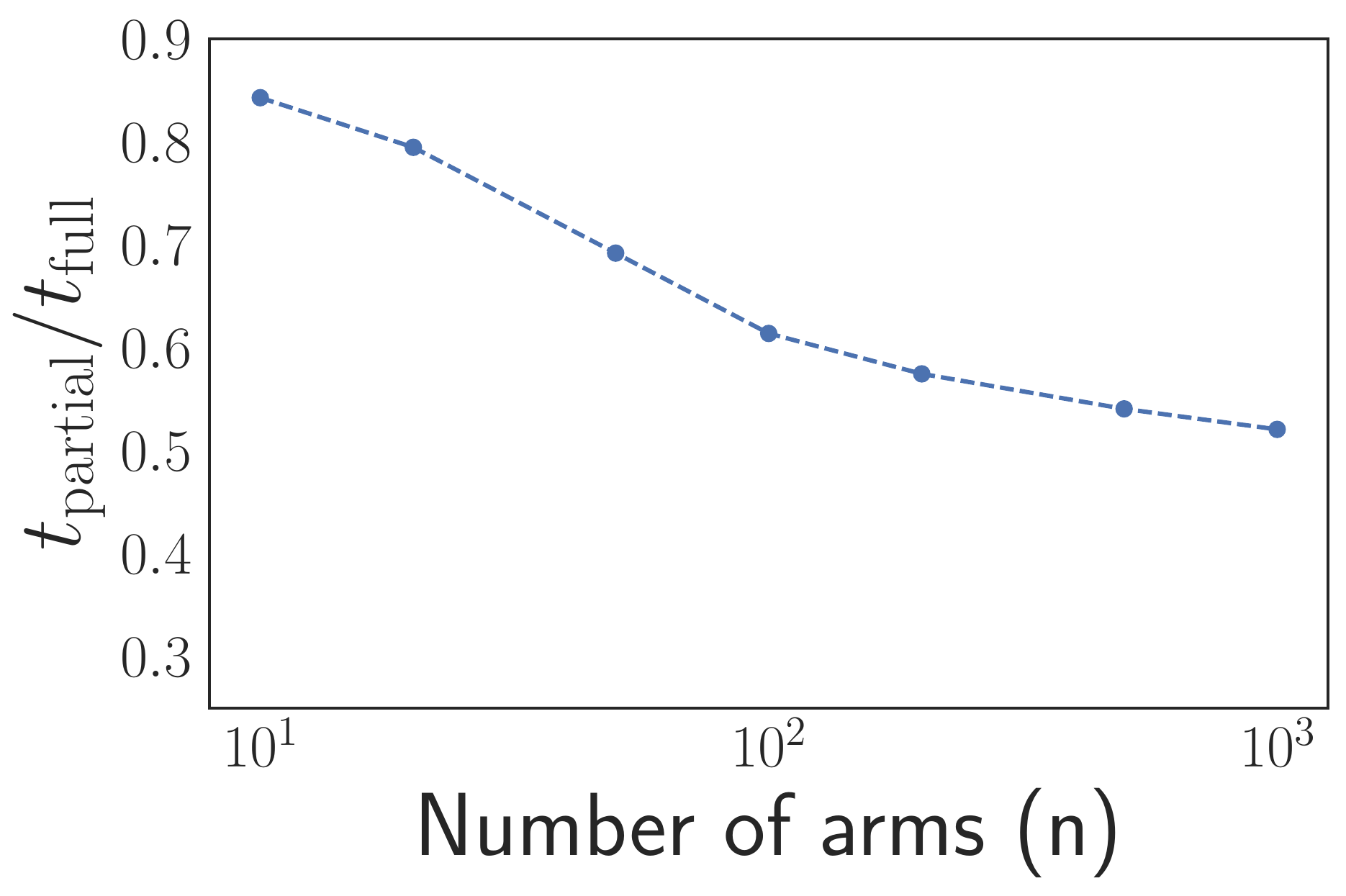}
\caption{Number of arms - Free means}\label{fig:exp_free_means}
\end{subfigure}
~
\begin{subfigure}[b]{0.315\textwidth}
\centering
\includegraphics[width=\textwidth]{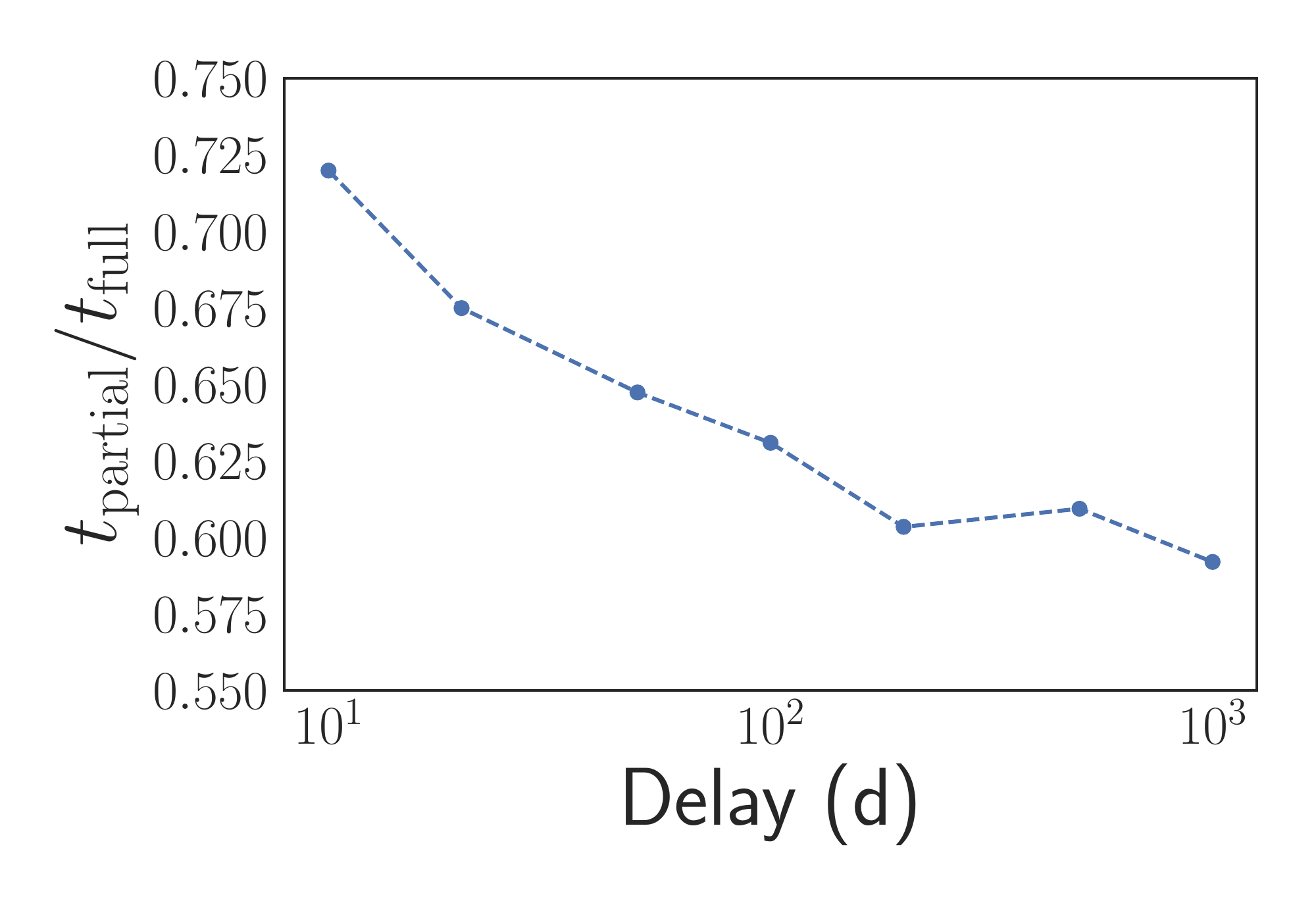}
\caption{Delay}\label{fig:exp_delay}
\end{subfigure}

\caption{Synthetic experiments evaluating performance. \textbf{Top:} sequential. \textbf{Bottom:} parallel. Lower is better.}\label{fig:exp}
\end{figure*}
\subsection{Biased noisy partial feedback}

The partial feedback at the intermediate time steps before a  full delayed feedback can also correspond to biased estimates of the full delayed feedback. Although the bias for the arms is unknown, it can be estimated empirically based on differences in the full delayed feedback and the partial feedback at the corresponding intermediate time steps. Formally, we assume the bias for a particular arm is an unknown constant $b_i$ and derive the following LIL bounds.
\begin{proposition}
\label{thm:lil_noisy_biased}
Let $\{Y_{i,t_1+1}, Y_{i,t_1+2}, \ldots, Y_{i,t_1+D_1}$, $Y_{i,t_2+1}, \ldots, Y_{i,t_2+D_2} \ldots\}$ denote the partial feedback sequences
for the pulls of an arm $i$ started at time steps $t_{1}, t_{2}, \ldots$ and delays $D_1, D_2, \ldots$ with bias $b_i$. Then, under the distributional assumptions on the partial feedback (see Section~\ref{sec:background}) for any $F\in \mathbb{N}\backslash\{1\}$, $P \in [1, D_F]$, $\delta_f > 0, \delta_p> 0, \delta_b> 0$, we have with probability $1-\delta_f-\delta_p-\delta_b$:
\begin{align}\label{eq:lil_noisy_biased}
&\left \vert \frac{1}{F}\left[ \sum_{f=1}^{F-1} X_{i, t_f} + \frac{1}{P} \sum_{p=1}^P \left(Y_{i, t_F+p} - Z_{i,F} \right) \right]- \mu_i \right \vert \nonumber \\
&\text{\resizebox{\hsize}{!}{$\leq C\left(\sigma_i, F, \nicefrac{\delta_f}{n}\right) + \frac{1}{F} \left [C\left(\sigma_i^{(p)}, P, \nicefrac{\delta_p}{n}\right) 
+ C\left(\sigma_i^{(p)}, F-1, \nicefrac{\delta_b}{n}\right)\right ]$}}\\
&\forall i \in [1,n] \text{where  \resizebox{.7\hsize}{!}{$Z_{i,F}=\frac{1}{F-1}\sum_{f=1}^{F-1}\left(\frac{\sum_{p=1}^{D_f-1} Y_{i, t_f+p}}{D_f-1} -  X_{i, D_f-1}\right)$}}. \nonumber
\end{align}
\end{proposition}

Comparing Eq.~\eqref{eq:lil_noisy_biased} with Eq.~\eqref{eq:lil_full_delayed} by constraining $\delta=\delta_f+\delta_p+\delta_b$, we see that the mean estimator takes into account the partial feedback as before but also has a bias correction term. The bias correction term is an empirical average of the biases observed from the past full delayed feedback. This correction has the effect of introducing additional uncertainty (third term in the RHS) and we need at least one full feedback to estimate the bias before we can use the above bound. The corresponding racing algorithm runs similar to Algorithm~\ref{alg:seqracingnoisypf} with the key difference being that the mean estimator corresponds to the minimum of the confidence bounds in Eq.~\eqref{eq:lil_full_delayed}  and Eq.~\eqref{eq:lil_noisy_biased}, where the RHS of Eq.~\eqref{eq:lil_noisy_biased} is specified for the optimal $\delta_f^\ast, \delta_p^\ast, \delta_b^\ast$ minimizing the expression under the constraint due to $\delta$. We defer the pseudocode for this setting to the Appendix (see Algorithm~\ref{alg:biased_seqracingnoisypf}).

\section{PARALLEL MAB}\label{sec:parallel}
In parallel MAB, an agent has the additional ability to ``accumulate'' bulk information by controlling a batch of arm pulls. We extend the $(b, r)$ setting proposed in \cite{jun2016top} where the agent is allowed to run at most $b$ arm pulls in parallel at any given time step with an upper limit $r$ on the number of pulls of each arm. 

Even the full delayed feedback setting becomes interesting, as the agent can exploit information from arm pulls which have finished running in parallel to accept/reject delayed arm pulls that are still running thereby avoiding the pitfalls of long delays. The pseudocode for the proposed batch racing algorithm with full delayed feedback is given in Algorithm~\ref{alg:batchracingfulldf}. At every time step, an agent pulls a batch of arms with the least pull count $N_i$ that obeys the $(b, r)$ constraints
(Lines~\ref{line:constraintstart}-\ref{line:constraintend}). Whenever we obtain at least one full delayed feedback, we can update our arm sets as per the racing criteria (Lines~\ref{line:updracingstart}-\ref{line:updracingend}). 

The algorithms for the noisy partial feedback settings discussed in Section~\ref{sec:single} can be extended for parallel MAB in a similar manner and are skipped here to keep the presentation clean. 
The theoretical analysis of the batch MAB setting in \cite{jun2016top} builds on the analysis of standard MAB in ways independent of the choice of LIL bounds and hence, a merged analysis for delayed batch MAB using the LIL bounds for delayed feedback (as in Propositions~\ref{thm:lil_noisy} and~\ref{thm:lil_noisy_biased}) suggests a reduction factor of $b$ in the corresponding upper bounds.

\section{EXPERIMENTS}
We empirically validated the proposed algorithms on a simulated setting and two real world datasets. All experiments use an error probability of $\delta = 0.05$ and we observed that in each case, the algorithm obtains the desired confidence level empirically. For the parallel MAB setting, we set $b=r=10$.
\begin{figure}[t]
\centering
\begin{subfigure}[b]{0.8\columnwidth}
\includegraphics[width=\textwidth]
{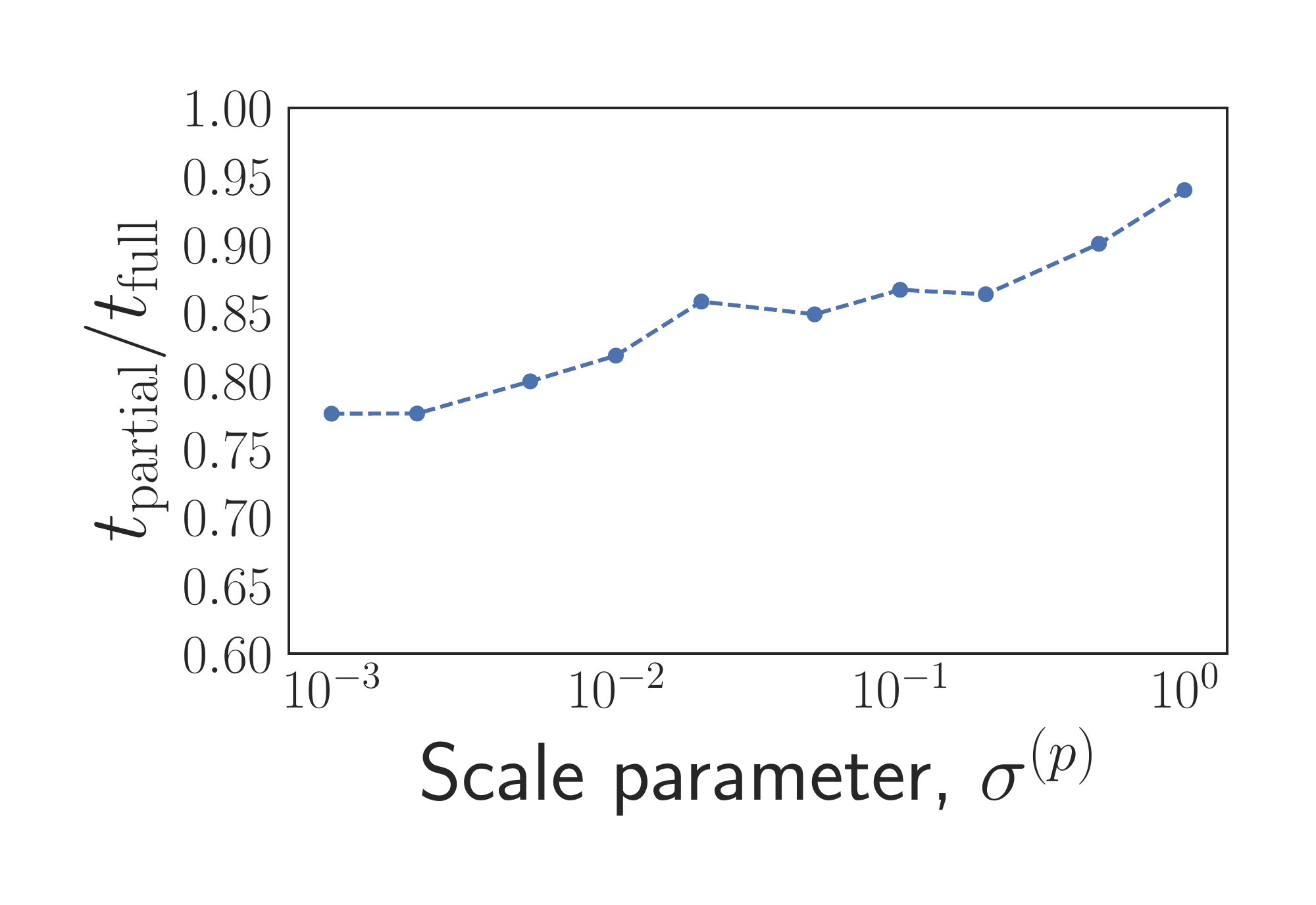}
\caption{Sequential}
\end{subfigure}
\begin{subfigure}[b]{0.8\columnwidth}
\includegraphics[width=\textwidth]
{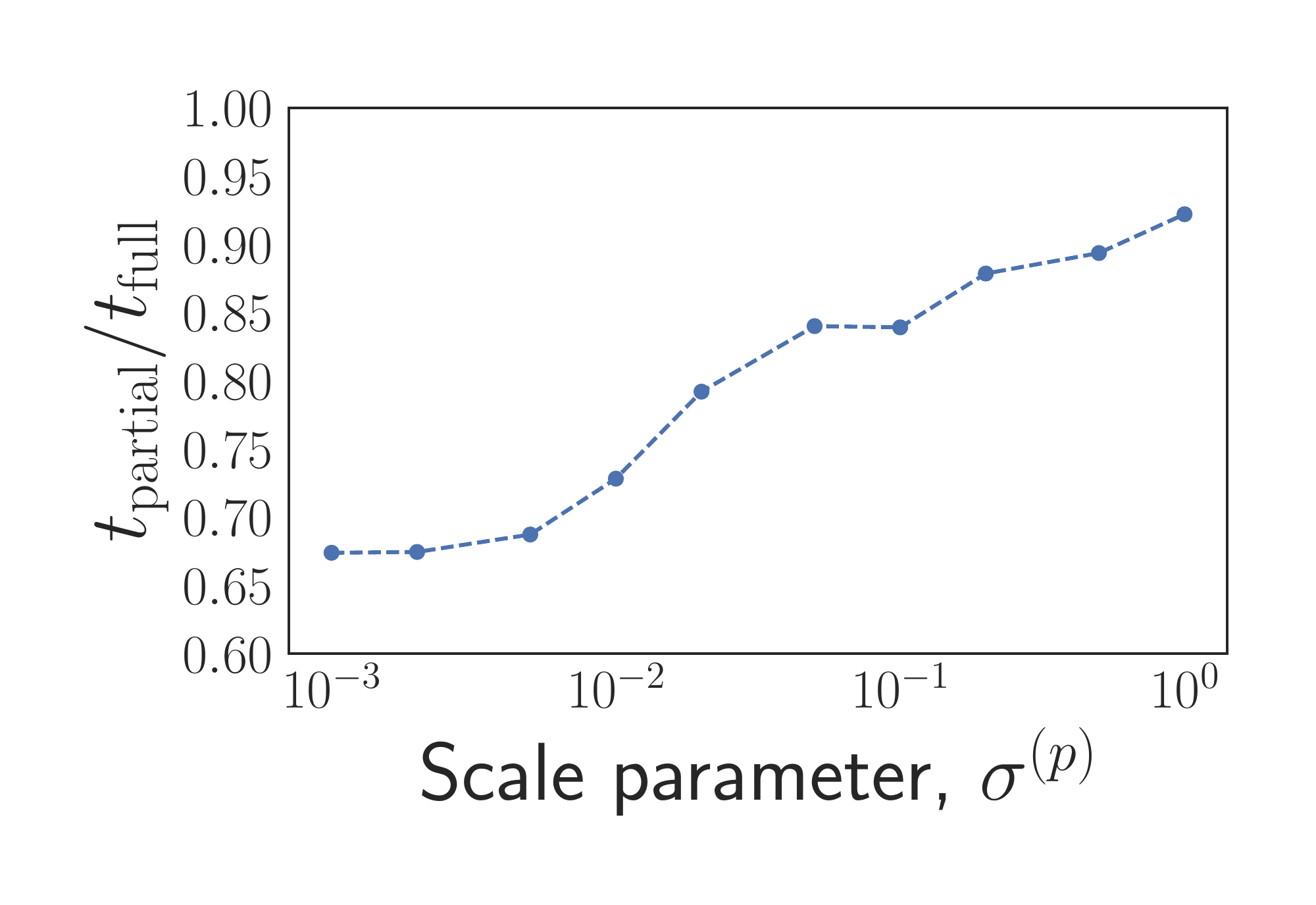}
\caption{Parallel}
\end{subfigure}
\caption{Experiments on battery charging.}\label{fig:battery}
\end{figure}

\subsection{Simulated data}
We performed an ablation study of the proposed algorithms for sequential and parallel MAB under different settings of delayed feedback. All experiments were repeated for $100$ random runs such that the standard errors are vanishingly small and the number of top arms to be identified, $k$ is set to $0.2n$. We quantify improvement as the ratio (=$\nicefrac{t_{\mathrm{partial}}}{t_{\mathrm{full}}}$) of the time taken by Algorithm~\ref{alg:seqracingnoisypf} or its parallel MAB extension (\textit{i.e.}, $t_{\mathrm{partial}}$) and the time taken by a full delayed feedback racing procedure (\textit{i.e.}, $t_{\mathrm{full}}$). We evaluate performance as a function of the following problem parameters.

\paragraph{Number of arms.} To analyze the difference in performance as a function of the number of arms ($n$), we further consider two distribution of means. 

In the \textit{bounded means} case, we set the means of the arms as $\mu_i = c - (i/n)^{\tilde{c}}$ for any choice of constants $c$ and $\tilde{c}>0$. Hence, the range of the means does not vary with $n$. In Figure~\ref{fig:exp_bounded_means}, we observe that accounting for unbiased partial feedback can give gains of up to 25\% and 40\% for the sequential and parallel MAB when the number of arms is low. The gains are reduced when the number of arms is large, which suggests that partial feedback is less advantageous in scenarios where a large number of full pulls are required for disambiguating very closely spaced means. 

In the \textit{free means} case, we set the means of the arms as $\mu_i = c - \tilde{c}i$ for any choice of constants $c$ and $\tilde{c}>0$. Here, the range of the means increases with $n$. From the results in Figure~\ref{fig:exp_free_means}, we observe that the gains due to partial feedback improve as the number of arms increases. This suggests that when the relative separation in means between the arms is fixed, Algorithm~\ref{alg:seqracingnoisypf} and its parallel MAB extension quickly eliminate arms with extreme means (very high or very low) unlike the racing algorithms that wait for full delayed feedback.

\paragraph{Delay.} 
Here, we fix $n=100$ and vary the delay of the arms.
For all settings of the delay in Figure~\ref{fig:exp_delay}, Algorithm~\ref{alg:seqracingnoisypf}  and its parallel MAB extension require a significantly lower fraction of the time with the lowest ratios observed to be $0.59$ and $0.57$ for sequential and parallel MAB respectively. While we did not see much variation in improvements for sequential MAB, the improvements are better for longer delays in the case of parallel MAB.

\subsection{Policy search for fast battery charging}

For any given battery chemistry, the charging (and discharging) policy has a significant impact on the lifetime of the cells. 
However, a single run of a particular policy however takes months to complete since every cell needs to be repeatedly charged and discharged until the end of its lifetime. Hence, delayed feedback can significantly slow down the search procedure. The true, unknown reward for any arm (charging policy) is stochastic and corresponds to the lifetime of the battery~\citep{harris2017failure,baumhofer2014production,schuster2015lithium}.\footnote{Formally, the lifetime of the cell is defined to be the number of cycles until a battery reaches $80\%$ of its original capacity at which point a battery is considered dead.}

We model the search for the best charging policy  for the Li-ion battery chemistry as a best arm identification problem in a stochastic MAB with $n=40$ arms, $k=1$. The true mean cycle life, cell-to-cell variances, and delays are obtained from a battery charging simulator~\citep{moura2017battery,perez2016optimal}. While a battery cell undergoes charging and discharging, we can additionally monitor key indicators such as voltage, temperature, and internal resistance. Predictive models of lifetime based on these factors is an active area of research, and can serve the purpose of partial feedback estimator~\citep{burns2013predicting,dubarry2017state}. We assume the existence of such an estimator and test the robustness of our algorithm by evaluating the relative improvements obtained from Algorithm~\ref{alg:seqracingnoisypf} on varying the noise $\sigma_i^{(p)}$ associated with the partial feedback. The results are shown in Figure~\ref{fig:battery}. 
When the estimator is ``trustworthy" (low $\sigma_i^{(p)}$), we can achieve improvements of up to 35\% in the number of experiments required.
As expected, the gains diminish for poorer models of partial feedback in which case the algorithm can choose to ignore the noisy feedback. 

\subsection{Hyperparameter optimization for mixed integer programming}

The CPLEX solver\footnote{\url{https://www.ibm.com/software/commerce/optimization/cplex-optimizer/index.html}} for mixed integer programming has a host of hyperparameters, including options to switch on or off different \textit{cut} strategies employed by the solver during the search process. We model the task of finding the best cut strategy as a stochastic MAB problem with $n=32$ arms (\textit{i.e.}, cut strategies), $k=1$. The performance is measured on CORLAT, a benchmark set of $2,000$ (maximization) mixed integer linear programming instances derived from real world data used for the construction of a wildlife corridor for grizzly bears in the Northern Rockies region~\citep{gomes2008connections,hutter2010automated}. The true mean for each arm is the average of lower bounds attained by the cut strategy on the feasible instances in the dataset under specified time and resource constraints per instance ($10$ seconds on $1$ core). Every pull of an arm  corresponds to running a cut strategy on a sampled problem instance.

Instead of waiting for the solver to completely solve (or time out) a sampled problem instance, we can save computation by using partial feedback about the search process. In particular, the solver outputs the best integral lower bound (LB) and real valued upper bound (UB) found after executing each cut during search. The final output of the solver is the best lower bound. 
To obtain an unbiased partial feedback estimator, we use a training subset of $500$ instances to learn a linear model that predicts the final lower bound for a given input instance based on the intermediate lower and upper bounds. The best arm identification algorithms are tested on the remaining instances in the dataset. Conditioned on a problem instance, the uncertainty associated with the partial feedback, $\sigma_i^{(p)}$ is given by $(UB - LB)/2$ and shrinks with an increase in the time steps elapsed. Note that the delays are not fixed  and depend on both the cut strategy and the problem instance under consideration. We directly report the final results: the percentage reduction in time taken by the unbiased partial feedback scenarios over full delayed feedback is $80.8\%$ and $87.6\%$ for sequential and parallel MAB respectively stressing the importance of partial feedback for this particular application scenario.

\section{RELATED WORK}
Early work in pure exploration is attributed to \cite{bechhofer1958sequential} and \cite{paulson1964sequential} 
who studied this problem 
in the context of optimal experimental design.
Modern day literature can be categorized into either the \emph{fixed budget} or the \emph{fixed confidence} settings. Algorithms for the fixed budget setting strive to maximize the probability of identifying the top-$k$ arms~\citep{audibert2010best,bubeck2013multiple,kaufmann2015complexity}. In the fixed confidence setting, which is the one we consider in this paper, 
the goal is to minimize the number of pulls to attain a target confidence~\citep{maron1994hoeffding,bubeck2009pure}. See \citet{gabillon2012best} for a unified treatment of the two settings.

Algorithms for the fixed confidence setting can be broadly classified into racing style procedures which sample arms uniformly and eliminate sub-optimal arms~\citep{maron1994hoeffding,even2002pac} and the UCB/LUCB style procedures which adaptively sample arms without explicit elimination. We direct the reader to the excellent survey by \cite{jamieson2014best} that summarizes the major advancements in the analysis of the sample complexity of these algorithms. Algorithmic generalizations of the best arm identification include top-$k$ identification~\citep{heidrich2009hoeffding} and the parallel MAB settings for batch arm pulls~\citep{perchet2015batched, jun2016top,wu2015identifying} among others.

While the delayed feedback framework we propose is novel to the pure exploration problem, online learning with delays has been studied previously in the regret minimization setting~\citep{weinberger2002delayed,joulani2013online,desautels2014parallelizing}. 
In particular, algorithms designed particularly for hyperparameter optimization have enjoyed great success. \cite{krueger2015fast} proposes a modified cross-validation procedure performed on increasing subsets of data coupled with a sequential testing strategy to eliminate the poor parameter configurations early on. \cite{jamieson2016non} and \cite{li2016hyperband} recently proposed algorithms for hyperparameter optimization based on non-stochastic MAB. Here, the arms correspond to hyperparameter configurations, and a pull is equivalent to observing a fixed sequence of losses. 

For many real-world problems, we have access to a 
shared structure across arms that makes the overall problem amenable to Bayesian optimization techniques~\citep{snoek2012practical,eggensperger2013towards,snoek2015scalable,feurer2015initializing,mitchBOpt16,mitchBOpt16IPAC}. Combining the LIL bounds we proposed for noisy partial feedback with Bayesian multi-armed bandits~\citep{srinivas2010gaussian,krause2011contextual,hoffman2013exploiting} is a promising extension we are pursuing for our on-going real world application relating to efficient search of fast charging policies for Li-ion battery cells~\citep{ermon2012learning}.

\section{CONCLUSIONS}
We introduced a new general framework for pure exploration in stochastic multi-armed bandit problems with partial and delayed feedback. We provided efficient algorithms for solving specific instantiations of our framework that can naturally model real world scenarios, especially in the context of optimal experimental design. 
We leave as future work the problem of identifying information-theoretic lower bounds on the sample complexity of the new pure exploration problems we formulated. Extension of our framework to the fixed budget setting is another interesting direction for future work.
\section*{ACKNOWLEDGEMENTS}
We are thankful to Neal Jean and Daniel Levy for helpful comments on early drafts. This research has been supported by a Microsoft Research PhD fellowship in machine learning for the first author, NSF grants \#1651565, \#1522054, \#1733686, Toyota Research Institute, Future of Life Institute, Precourt Institute for Energy, and Intel.

\bibliographystyle{plainnat}
\bibliography{refs}

\onecolumn
\section*{Appendices}
\begin{appendices}
\section{Unbiased noisy partial feedback}
\subsection{Proposition~\ref{thm:lil_noisy}}
\begin{proof}
By Lemma~\ref{thm:lil_adaptive} applied to $X_{i,t_1}, X_{i,t_2} , \ldots$ for an arm $i$ for $F$ full delayed feedback, we have w.p. $1-\nicefrac{\delta_f}{n}$:
\begin{align}\label{eq:lil_biased_1}
\left \vert \frac{1}{F} \sum_{f=1}^F X_{i, t_f} - \mu_i \right \vert \leq C\left(\sigma_i, F, \nicefrac{\delta_f}{n}\right).
\end{align}
For any $a$, $E[Y_{i, t_F+p}\vert X_{i,t_F}=a] =a$, and $E[Y_{i, t_F+p} -a \vert X_{i,t_F}=a] =0$. Conditioned on $X_{i,t_F}=a$, $(Y_{i, t_F+p} -a) \vert (X_{i,t_F}=a)$ is sub-Gaussian by assumption.

Therefore, conditioned on $X_{i,t_F}=a$, 
by Lemma~\ref{thm:lil_adaptive} applied to $\frac{1}{P}\sum_{p=1}^P (Y_{i, t_F+p}|X_{i,t_F}=a) - a$ for an arm $i$ computed using $P$ partial feedback for the $F$-th pull, we have w.p. $1-\nicefrac{\delta_p}{n}$:
\begin{align}\label{eq:lil_biased_44}
\left \vert \frac{1}{P} \sum_{p=1}^P Y_{i, t_F+p} - a \right \vert \leq C\left(\sigma_i^{(p)}, P, \nicefrac{\delta_p}{n}\right).
\end{align}
Given that the result does not depend on the value $a$, we have:
\begin{align}\label{eq:lil_biased_2}
\left \vert \frac{1}{P} \sum_{p=1}^P \left(Y_{i, t_F+p}|X_{i,t_F}\right) - X_{i,t_F} \right \vert \leq C\left(\sigma_i^{(p)}, P, \nicefrac{\delta_p}{n}\right).
\end{align}

From a union bound Eq.~\eqref{eq:lil_biased_1} and Eq.~\eqref{eq:lil_biased_2}, we have w.p. $1-\nicefrac{\delta_f}{n}-\nicefrac{\delta_p}{n}$:
\begin{align}\label{eq:lil_noisy_4}
\left \vert \frac{1}{F}\left[ \sum_{f=1}^{F-1} X_{i, t_f} + \frac{1}{P} \sum_{p=1}^P Y_{i, t_F+p}\right]- \mu_i \right \vert \leq C\left(\sigma_i, F, \nicefrac{\delta_f}{n}\right) + \frac{1}{F} C\left(\sigma_i^{(p)}, P, \nicefrac{\delta_p}{n}\right).
\end{align}

Union bounding Eq.\eqref{eq:lil_noisy_4} over all arms, we have w.p. $1-\delta_f-\delta_p$:
\begin{align}\label{eq:lil_noisy_5}
\left \vert \frac{1}{F}\left[ \sum_{f=1}^{F-1} X_{i, t_f} + \frac{1}{P} \sum_{p=1}^P Y_{i, t_F+p}\right]- \mu_i \right \vert \leq C\left(\sigma_i, F, \nicefrac{\delta_f}{n}\right) + \frac{1}{F} C\left(\sigma_i^{(p)}, P, \nicefrac{\delta_p}{n}\right) \; \forall i \in [1,n]
\end{align}
finishing the proof.
\end{proof}

\subsection{Theorem~\ref{thm:seqracingnoisypf_sample}}

At any given time $t\geq 1$, $F \in \mathbb{N}, P\in [1, D_F]$, we observe $F-1$ full feedback, $X_{i, t_{1:F-1}}$ for an arbitrary arm $i \in [1,n]$. Accordingly, we have the following two cases to consider as per Algorithm~\ref{alg:seqracingnoisypf}. 
\begin{itemize}
\item \textbf{Case (a):} $C(\sigma_i, F-1, \nicefrac{\delta}{n}) < C(\sigma_i, F, \nicefrac{\delta_f^\ast}{n}) + \frac{1}{F} C(\sigma^{(p)}_i, P, \nicefrac{\delta_p^\ast}{n})$
\begin{align*}
\widehat{\mu}_{i}&=\frac{1}{F-1}\sum_{f=1}^{F-1} X_{i,t_f} \\
C_{i} &= C\left(\sigma_i, F-1, \nicefrac{\delta}{n}\right).
\end{align*}
\item \textbf{Case (b):} otherwise
\begin{align*}
\widehat{\mu}_{i}&=\frac{1}{F}\left[\sum_{f=1}^{F-1} X_{i,t_f} +  \frac{1}{P}\sum_{l=1}^P Y_{i, t_{F}+l} \right] \\
C_{i} &= C(\sigma_i, F, \nicefrac{\delta_f^\ast}{n}) + \frac{1}{F} C(\sigma^{(p)}_i, P, \nicefrac{\delta_p^\ast}{n}).
\end{align*}
\end{itemize}

Define $\mathcal{E}_i =\{\forall t \geq 1, \vert \widehat{\mu}_{i} - \mu_i \vert\leq C_i \}$ be the event that the lower and upper confidence bounds of arm $i$ trap the true mean $\mu_i$ for all $t \geq 1$ where $\widehat{\mu}_{i}$ and $C_i$ are chosen as described above at time $t$. Let $S_t, A_t, R_t$ denote the set of surviving, accepted, and rejected arms at time $t$. We can then state and prove the following lemma.

\begin{lemma}\label{thm:seqracingdf_correctness}
Assume $\mathcal{E}_i$ holds for an arbitrary arm $i \in S_t$ and $i \not \in S_{t+1}$. Then, the following statements hold:
\begin{itemize}
\item $i \in A_{t+1}$ if $i \leq k$.
\item $i \in R_{t+1}$ if $i > k$.
\end{itemize}
\end{lemma}
\begin{proof}

By definition, $S_{t+1} \cup A_{t+1} \cup R_{t+1} = S_{t}$. Recursing over $t, t-1, ...0$, we note that $S_{t+1} \cup A_{t+1} \cup R_{t+1} = \{1,2,\ldots,n \}$. Since the lemma assumes that arm $i\not \in S_{t+1}$, either $i \in A_{t+1}$ or $i \in R_{t+1}$. 

We will prove the first statement of the lemma by contradiction. For an arbitrary $i \leq k$, let us assume $i \in R_{t+1}$. This implies that $UCB_i < \max_{j \in S_t}^{(k)} LCB_j$. Since by assumptions on the lemma the lower and upper confidence bounds of any arm trap its true mean, we have $UCB_i  \geq \mu_i$ and $\max_{j \in S_t}^{(k)} LCB_j \leq \mu_k$. Hence, we obtain $\mu_i < \mu_k$ which is a contradiction since $i\leq k$. The second statement holds true by symmetry.
\end{proof}

Since both Proposition~\ref{thm:lil_noisy} and Eq.~\eqref{eq:lil_full_delayed}
hold true w.p. at least $1-\delta/n$ for all arms, we get that $\cap_{i=1}^n \mathcal{E}_i$ holds true w.p. at least $1-\delta$ (union bound) regardless of the set of $\{\widehat{\mu}_i\}_{i=1}^n$ and $\{C_i\}_{i=1}^n$ picked by the algorithm. Combining the union bound with Lemma~\ref{thm:seqracingdf_correctness}, the algorithm outputs the top-$k$ set w.p. at least $1-\delta$ if it terminates.

\section{Biased noisy partial feedback}
\subsection{Proposition~\ref{thm:lil_noisy_biased}}
\begin{proof}
By Lemma~\ref{thm:lil_adaptive} applied to $X_{i,t_1}, X_{i,t_2} , \ldots$ for an arm $i$ for $F$ full delayed feedback, we have w.p. $1-\nicefrac{\delta_f}{n}$:
\begin{align}\label{eq:lil_noisy_biased_1}
\left \vert \frac{1}{F} \sum_{f=1}^F X_{i, t_f} - \mu_i \right \vert \leq C\left(\sigma_i, F, \nicefrac{\delta_f}{n}\right).
\end{align}

For any $a$, $E[Y_{i, t_F+p}\vert X_{i,t_F}=a] =a + b_i$, and $E[Y_{i, t_F+p} -a-b_i|X_{i,t_F}=a] =0$. Conditioned on $X_{i,t_F}=a$, $(Y_{i, t_F+p} -a-b_i)\vert (X_{i,t_F}=a)$ is sub-Gaussian by assumption.
Therefore, conditioned on $X_{i,t_F}=a$, by Lemma~\ref{thm:lil_adaptive} applied to $\frac{1}{P} \sum_{p=1}^P \left(Y_{i, t_F+p}-b_i\right) - X_{i,t_F}$ for the (incomplete) $F$-th pull of an arm $i$ with $P$ partial feedback, we have w.p. $1-\nicefrac{\delta_p}{n}$:
\begin{align}\label{eq:lil_noisy_biased_2}
\left \vert \frac{1}{P} \sum_{p=1}^P \left(Y_{i, t_F+p}-b_i\right) - X_{i,t_F} \right \vert \leq C\left(\sigma_i^{(p)}, P, \nicefrac{\delta_p}{n}\right).
\end{align}

Now, consider the $F-1$ random variables for all $f \in [1, F-1]$:
\begin{align}\label{eq:rv_noisy_biased_1}
\frac{\sum_{p=1}^{D_f-1} Y_{i, t_f+p}}{D_f-1} - X_{i,f}.
\end{align}
The random variables in \eqref{eq:rv_noisy_biased_1} are all sub-Gaussian with mean $b_i$ and scale parameter $\sigma_i^{(p)}$. Hence, applying LIL on these random variables conditioning on $b_i$, we have w.p. $1-\nicefrac{\delta_b}{n}$:
\begin{align}\label{eq:lil_noisy_biased_3}
\left \vert \frac{1}{F-1} \sum_{f=1}^{F-1} \left(\frac{\sum_{p=1}^{D_f-1} Y_{i, t_f+p}}{D_f-1} -  X_{i, D_f-1}\right) - b_i \right \vert \leq C\left(\sigma_i^{(p)}, F-1, \nicefrac{\delta_b}{n}\right).
\end{align}

From a union bound of Eq.~\eqref{eq:lil_noisy_biased_1} and Eq.~\eqref{eq:lil_noisy_biased_2}, we have w.p. $1-\nicefrac{\delta_f}{n}-\nicefrac{\delta_p}{n}$:
\begin{align}\label{eq:lil_noisy_biased_4}
\left \vert \frac{1}{F}\left[ \sum_{f=1}^{F-1} X_{i, t_f} + \frac{1}{P} \sum_{p=1}^P \left(Y_{i, t_F+p} -b_i \right)\right]- \mu_i \right \vert \leq C\left(\sigma_i, F, \nicefrac{\delta_f}{n}\right) + \frac{1}{F} C\left(\sigma_i^{(p)}, P, \nicefrac{\delta_p}{n}\right).
\end{align}

From a union bound of Eq.~\eqref{eq:lil_noisy_biased_3} and Eq.~\eqref{eq:lil_noisy_biased_4}, we have w.p. $1-\nicefrac{\delta_f}{n}-\nicefrac{\delta_p}{n}-\nicefrac{\delta_b}{n}$:
\begin{align}\label{eq:lil_noisy_biased_5}
\left \vert \frac{1}{F}\left[ \sum_{f=1}^{F-1} X_{i, t_f} + \frac{1}{P} \sum_{p=1}^P \left(Y_{i, t_F+p} - \frac{1}{F-1}\sum_{f=1}^{F-1}\left(\frac{\sum_{p=1}^{D_f-1} Y_{i, t_f+p}}{D_f-1} -  X_{i, D_f-1}\right) \right)\right]- \mu_i \right \vert \nonumber \\
\leq C\left(\sigma_i, F, \nicefrac{\delta_f}{n}\right) + \frac{1}{F} C\left(\sigma_i^{(p)}, P, \nicefrac{\delta_p}{n}\right) + \frac{1}{F} C\left(\sigma_i^{(p)}, F-1, \nicefrac{\delta_b}{n}\right).
\end{align}

Finally, union bounding Eq.~\eqref{eq:lil_noisy_biased_5} over all arms, 
we have w.p. $1-\delta_f-\delta_p-\delta_b$:
\begin{align}\label{eq:lil_noisy_biased_6}
\left \vert \frac{1}{F}\left[ \sum_{f=1}^{F-1} X_{i, t_f} + \frac{1}{P} \sum_{p=1}^P \left(Y_{i, t_F+p} - \frac{1}{F-1}\sum_{f=1}^{F-1}\left(\frac{\sum_{p=1}^{D_f-1} Y_{i, t_f+p}}{D_f-1} -  X_{i, D_f-1}\right) \right)\right]- \mu_i \right \vert \nonumber \\
\leq C\left(\sigma_i, F, \nicefrac{\delta_f}{n}\right) + \frac{1}{F} C\left(\sigma_i^{(p)}, P, \nicefrac{\delta_p}{n}\right) + \frac{1}{F} C\left(\sigma_i^{(p)}, F-1, \nicefrac{\delta_b}{n}\right) \; \forall i \in [1, n]
\end{align}
finishing the proof.
\end{proof}

\subsection{Algorithm}

\begin{algorithm}[t]
   \caption{RacingBiasedPF (arm parameters $\{i, \sigma_i, \sigma^{(p)}_i\}_{i=1}^n$, top $k$, confidence $\delta$)}
   \label{alg:biased_seqracingnoisypf}
\begin{algorithmic}[1]
\State Initialize global time step $t=0$, surviving $S=\{i\}_{i=1}^n$, accepted $A=\{\}$, rejected $R=\{\}$.
\State Initialize per-arm full delayed feedback counter $F_i=0$, empirical means $\hat{\mu}_{i}=0$,  confidence bounds $LCB_i=-\infty$,  $UCB_i=\infty$ for all $i \in S$. 
\While{$S$ is not empty}
\While {$\mathrm{True}$}
\State Increment $t \leftarrow t+1$.
\State Collect partial feedback $Y_{a, t}$.
\State\label{line:biased_mean}  Update $\widehat{\mu}^{(p)}$ using $Y_{a, t}$ as per Proposition~\ref{thm:lil_noisy_biased}. 
\State Increment $P \leftarrow P + 1$. 
\State\label{line:biased_cb} Set $C^{(partial)} \leftarrow 
C\left(\sigma_a, F_a + 1, \nicefrac{\delta_f^\ast}{n}\right) + \frac{1}{F_a+1} \left [C\left(\sigma_a^{(p)}, P, \nicefrac{\delta_p^\ast}{n}\right) 
+ C\left(\sigma_a^{(p)}, F_a, \nicefrac{\delta_b^\ast}{n}\right)\right ]$
\State\label{line:biased_partial_start} Choose $\mathrm{FOrP} \leftarrow \mathrm{arg}\min \left(C(\sigma_a, F_a, \nicefrac{\delta}{n}), C^{(partial)}\right)$.
\State Update $C_a \leftarrow C(\sigma_a, F_a, \nicefrac{\delta}{n})$ if $\mathrm{FOrP}=F$ else $C^{(partial)}$.
\State Update $\widehat{\mu}_a \leftarrow  \widehat{\mu}^{(f)}$ if $\mathrm{FOrP}=F$ else $\frac{F_a\widehat{\mu}^{(f)} + \widehat{\mu}^{(p)}}{F_a + 1}$.
\State\label{line:biased_partial_end} Update $LCB_a, UCB_a$.
\State\label{line:biased_racing_elimination} $A, R, S \leftarrow \mathrm{UpdateArmSets}(A, R, S, k, \{LCB_i, UCB_i)\}_{i \in S})$.
\If {$P=D_{a, t_a}$ or $a \not \in S$}\label{line:biased_end_delay} 
\State Break \Comment{Pull on  termination/elimination}
\EndIf
\EndWhile 
\State\label{line:biased_get_new_arm} Pull arm $a$ where $a \leftarrow \mathrm{arg} \min_{a \in S} F_a$.
\State Initialize start $t_a\leftarrow t$, partial feedback counter $P = 0$, partial mean $\widehat{\mu}^{(p)} = 0$, full mean $\widehat{\mu}^{(f)} \leftarrow \widehat{\mu}_i$.
\EndWhile
\State \Return $A$
\end{algorithmic}
\end{algorithm}
We provide the pseudocode for the racing procedures with biased partial feedback in Algorithm~\ref{alg:biased_seqracingnoisypf}. As discussed previously, the algorithm is similar to Algorithm~\ref{alg:seqracingnoisypf} with key differences in the mean and confidence bound estimators in Line~\ref{line:biased_mean} and Line~\ref{line:biased_cb} respectively.

\end{appendices}

\end{document}